\documentclass[runningheads]{llncs}

\usepackage{eccv}

\usepackage{eccvabbrv}

\usepackage{graphicx}
\usepackage{booktabs}

\usepackage[accsupp]{axessibility}  

\usepackage{array}
\usepackage{multirow}
\usepackage{makecell}
\usepackage{pifont}

\usepackage{wrapfig}
\usepackage{caption}

\usepackage{amsmath}
\usepackage{amssymb}

\usepackage[pagebackref]{hyperref}

\begin{document}

\title{On the Structural Failure of Chamfer Distance\\in 3D Shape Optimization}


\author{Chang-Yong Song \and David Hyde}

\authorrunning{C.-Y.~Song and D.~Hyde}

\institute{Vanderbilt University, Nashville, TN, USA\\
\email{\{chang-yong.song, david.hyde.1\}@vanderbilt.edu}}

\maketitle
\begin{abstract}
Chamfer distance is the standard training loss for point cloud reconstruction, completion, and generation, yet directly optimizing it can produce worse Chamfer values than not optimizing it at all.
We show that this paradoxical failure is gradient-structural. The per-point Chamfer gradient creates a many-to-one collapse that is the unique attractor of the forward term and cannot be resolved by any local regularizer, including repulsion, smoothness, and density-aware re-weighting.
We derive a necessary condition for collapse suppression: coupling must propagate beyond local neighborhoods.
In a controlled 2D setting, shared-basis deformation suppresses collapse by providing global coupling; in 3D shape morphing, a differentiable MPM prior instantiates the same principle, consistently reducing the Chamfer gap across 20 directed pairs with a 2.5$\times$ improvement on the topologically complex dragon.
The presence or absence of non-local coupling determines whether Chamfer optimization succeeds or collapses.
This provides a practical design criterion for any pipeline that optimizes point-level distance metrics.
\end{abstract}

\section{Introduction}
\label{sec:intro}

3D Chamfer distance has been widely adopted as a standard metric~\cite{fan2017point, yuan2018pcn, achlioptas2018learning} and training loss for point cloud reconstruction, shape completion, and 3D generation.
Its known limitations (non-uniform distributions, surface holes, convergence instabilities) have been treated primarily as a metric design problem: density-aware~\cite{wu2021density}, contrastive~\cite{lin2023infocd}, and learnable~\cite{huang2024learnable} variants re-weight or restructure the distance, implicitly assuming that a better metric yields better optimization.
We show that the core failure is not a metric-level issue but a gradient-structural one.
Within any fixed nearest-neighbor (NN) assignment region, the forward Chamfer gradient has a unique attractor at the many-to-one collapse. This attractor persists under any local modification, including density re-weighting, repulsion, and smoothness---because all such modifications preserve the nearest-neighbor assignment structure.

In practice, directly minimizing Chamfer distance (Direct Chamfer Optimization, DCO) produces worse Chamfer values than a physics-only baseline that never optimizes Chamfer at all.

\begin{figure}[t!]
    \centering
    \includegraphics[width=\linewidth]{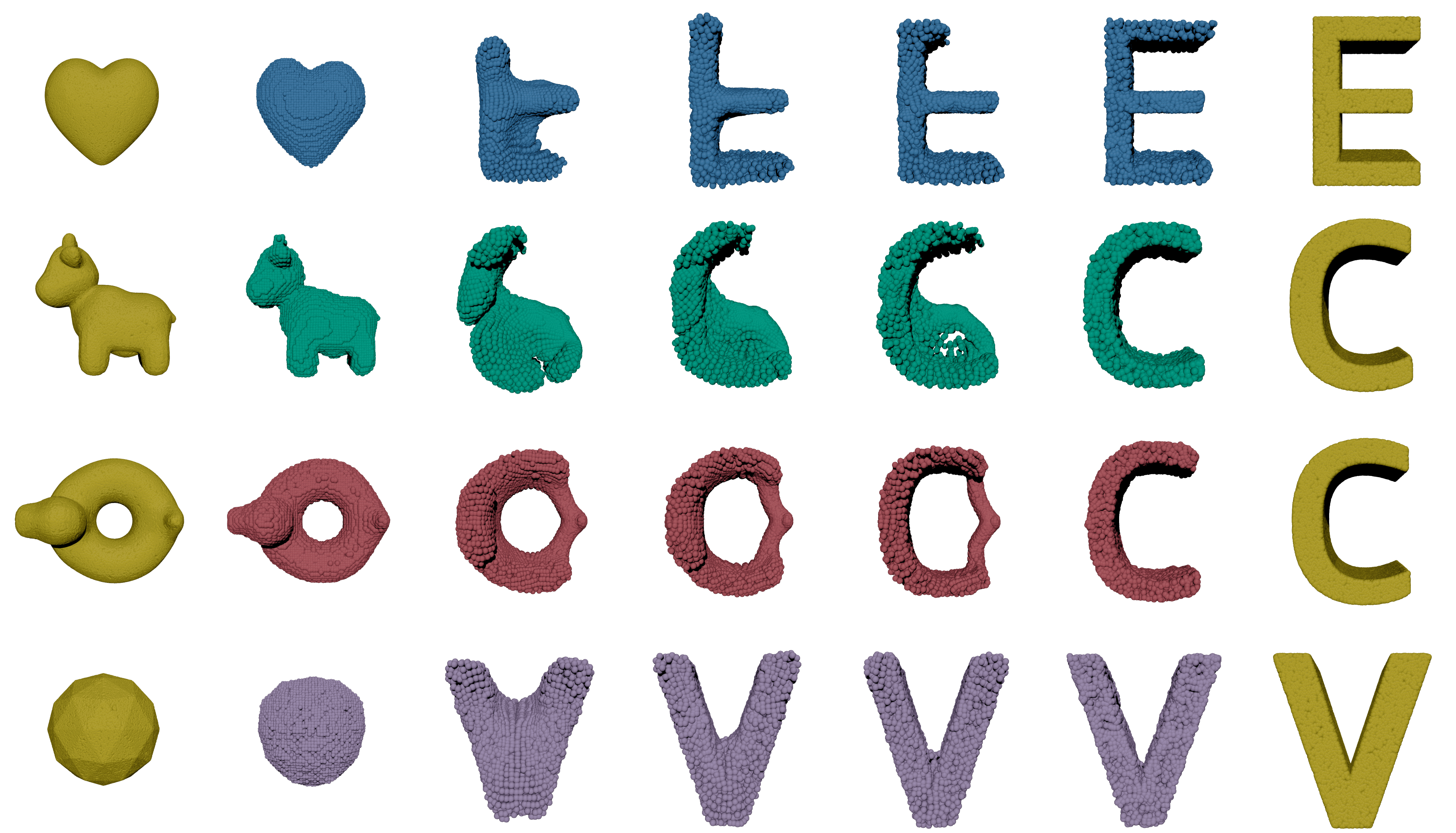}
    \caption{\textbf{Physics-guided Chamfer optimization on diverse source--target pairs.} Each row morphs a different source shape (Heart, Cow, Duck, Sphere) into a letter target. The differentiable physics prior provides non-local coupling, instantiating the design principle derived in Corollary~1, producing physically valid trajectories that faithfully capture the target geometry despite large topological differences between source and target
(e.g., the genus-1 torus-like mesh closing into C).}
    \label{fig:teaser}
\end{figure}

Figure~\ref{fig:teaser} previews the effect of global coupling on diverse source--target pairs: each source smoothly deforms into its letter target while maintaining volumetric integrity, even across large topological changes such as the genus-1 torus-like mesh morphing into C.

\begin{figure}[t!]
    \centering
    \includegraphics[width=\linewidth]{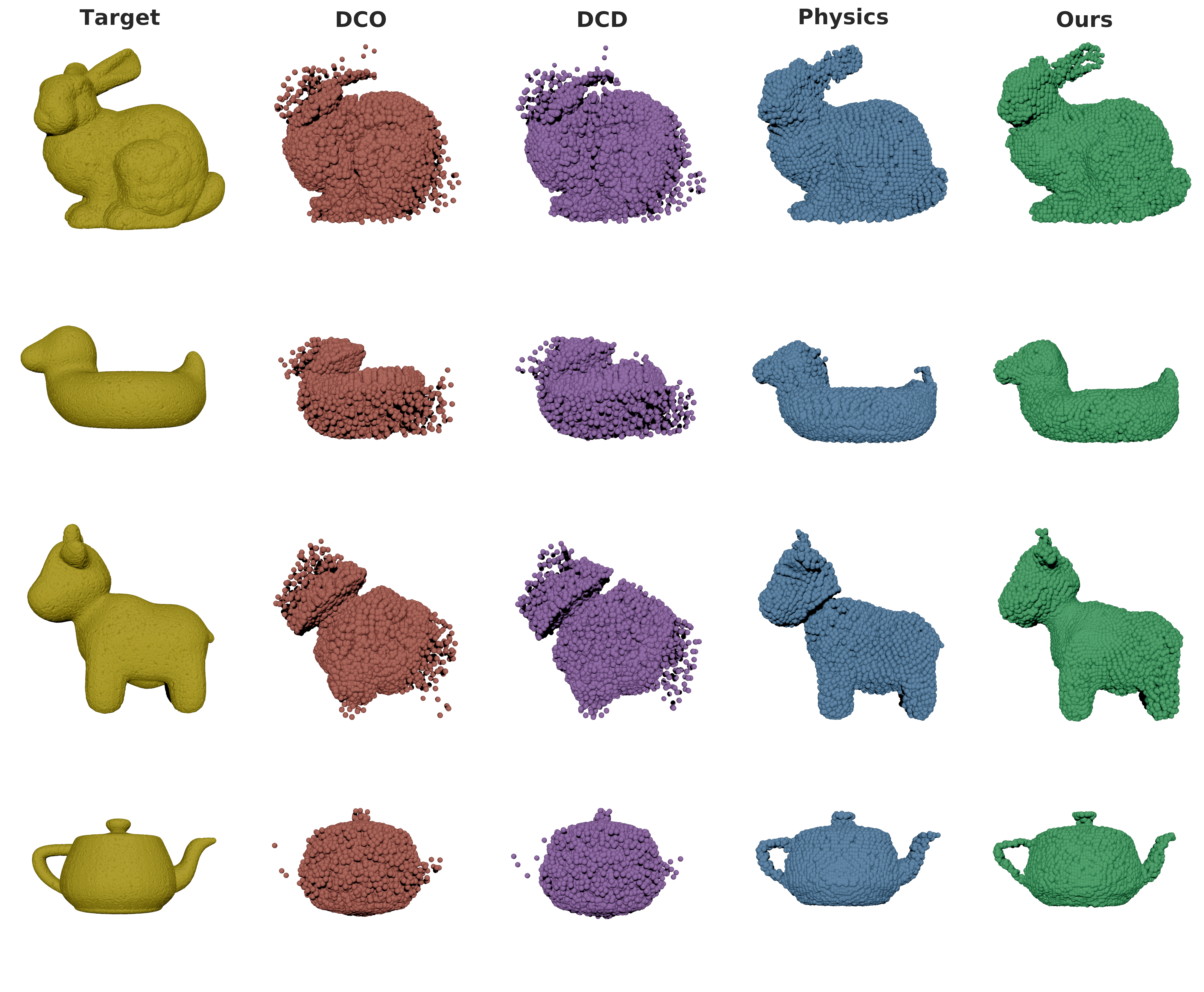}
    \caption{\textbf{Qualitative comparison on four target shapes.}
    DCO and DCD converge to structurally degraded volumetric shapes on all targets.
    The physics-only baseline preserves global plausibility but leaves a residual geometric gap.
    Our coupled physics--Chamfer method improves alignment while maintaining physically plausible volumetric structure.}
    \label{fig:4method}
\end{figure}

We take 3D shape morphing as a testbed for systematically examining these failure behaviors.
Shape morphing asks for a continuous trajectory from a source to a target shape~\cite{sorkine2007rigid,cosmo2020limp,eisenberger2021neuromorph,xu2025differentiable}; existing frameworks emphasize physical or structural losses but do not examine how they interact with point-level objectives like Chamfer distance.
Our primary diagnostic baseline (Direct Chamfer Optimization, DCO) optimizes point positions via the Chamfer gradient alone, without any morphing framework, physics prior, or temporal structure; the collapse it exhibits is therefore a property of the Chamfer objective itself, not an artifact of the morphing setting.

Unlike prior work that treats Chamfer failure as a metric design problem and proposes metric-level fixes, we show it is an optimization landscape problem: the collapse is a gradient attractor, not a measurement artifact.
Because the failure is gradient-structural rather than metric-level, redesigning the distance function cannot resolve it; the precise property required to escape collapse is coupling that propagates beyond local neighborhoods.
Our contributions separate into a general analysis of Chamfer distance (domain-agnostic) and a domain-specific validation using shape morphing:
\begin{enumerate}
    \item \textbf{Failure characterization (general).} We prove that many-to-one collapse is the unique attractor of the forward Chamfer gradient (Proposition~1), that the reverse term cannot separate collapsed points (Proposition~2), and that local regularizers cannot alter cluster-level drift (Proposition~3). Proposition~3 goes further: this is not an empirical limitation but a structural impossibility---no local regularizer, regardless of strength or form, can alter the cluster-level drift, because translational invariance guarantees that pairwise forces cancel at the centroid. These results hold for any point set optimized against a nearest-neighbor-based objective, regardless of application domain (Section~\ref{sec:analysis}).
    \item \textbf{Design principle for collapse suppression (general).} We derive a necessary condition that any successful remedy must satisfy: coupling must propagate beyond local neighborhoods to oppose the collapse attractor (Corollary~1). This rules out the entire class of local fixes---not just specific instances---and provides a testable design criterion without prescribing a particular implementation. We validate this principle across domains with a controlled 2D experiment (Section~\ref{sec:2d_experiment}) in which shared-basis deformation suppresses the collapse that per-point optimization and local repulsion cannot.
    \item \textbf{Verification via physics-guided morphing (domain-specific).} Using 3D shape morphing as a testbed, we implement the design principle through a differentiable Material Point Method (MPM) prior, the simplest available form of global coupling for particle systems, and verify Corollary~1 across 20 directed morphing pairs (Section~\ref{sec:exp_pairwise}), with a 2.5$\times$ improvement on the topologically complex dragon (Section~\ref{sec:dragon}). The minimal design isolates the effect of global coupling itself, confirming that the coupling property identified by our analysis, not architectural novelty, drives the improvement.
\end{enumerate}
The design principle (Corollary~1) applies wherever Chamfer distance is optimized; the specific implementation (MPM) is one natural choice for morphing, validated by the 2D experiment's confirmation that the principle generalizes beyond any single solver or domain (Section~\ref{sec:2d_experiment}).

\section{Related Work}
\label{sec:related}

\paragraph{Chamfer distance and its variants.}
Point cloud pipelines for reconstruction, completion, generation~\cite{qi2017pointnet,qi2017pointnet++,yuan2018pcn,yang2018foldingnet,groueix2018papier,achlioptas2018learning} and implicit surface learning~\cite{mildenhall2021nerf,wang2021neus,atzmon2020sal,niemeyer2020differentiable} rely on Chamfer distance as both a training loss and evaluation metric.
Known limitations have prompted density-aware~\cite{wu2021density}, learnable~\cite{huang2024learnable,urbach2020dpdist}, contrastive~\cite{lin2023infocd}, and curvature-aware~\cite{wang2018pixel2mesh,wen2019pixel2mesh++,gkioxari2019mesh} variants, as well as optimal-transport alternatives~\cite{rubner2000earth,nguyen2020distributional}.
These approaches treat the failure as a metric-level problem and propose improved distance functions.
In contrast, we show the failure is a gradient-structural property of nearest-neighbor-based objectives: local regularizers including repulsion~\cite{huang2009consolidation,yu2018ec}, smoothness~\cite{nealen2006laplacian}, volume preservation~\cite{huang2006subspace}, and normal consistency~\cite{hoppe1992surface,yu2018pu} are commonly combined with Chamfer distance (CD), but we prove that they cannot alter the cluster-level drift that drives many-to-one collapse (Proposition~3).

\paragraph{Point cloud optimization and regularization.}
Optimizing point sets against geometric objectives is central to surface reconstruction~\cite{kazhdan2006poisson,kazhdan2013screened}, shape completion~\cite{yuan2018pcn,xie2020grnet,wang2020cascaded}, and generative modeling~\cite{achlioptas2018learning,yang2019pointflow,luo2021diffusion}.
Repulsion-based spacing~\cite{huang2009consolidation,yu2018ec}, Laplacian smoothness~\cite{nealen2006laplacian,sorkine2004laplacian}, and volume preservation~\cite{huang2006subspace} are standard local regularizers aimed at improving distribution uniformity.
Optimal-transport objectives such as earth mover's distance~\cite{rubner2000earth} and sliced Wasserstein distance~\cite{nguyen2020distributional,cuturi2013sinkhorn} provide alternatives with different theoretical guarantees but at higher computational cost.
Our analysis complements these approaches by identifying a structural limitation shared by all local regularizers when combined with nearest-neighbor-based objectives.

\paragraph{Differentiable physics and shape morphing.}
Shape morphing spans ARAP deformation~\cite{sorkine2007rigid}, volume-preserving maps~\cite{aigerman2015orbifold}, latent-space interpolation~\cite{park2019deepsdf,mescheder2019occupancy,chen2019learning}, Hamiltonian dynamics~\cite{eisenberger2018divergence,eisenberger2021neuromorph}, neural flows~\cite{gupta2020neural,yang2019pointflow,chen2018neural}, and dynamic 3D Gaussian splatting~\cite{kerbl20233d,wu20244d,yang2024deformable,luiten2024dynamic}.
Differentiable physics frameworks~\cite{hu2019chainqueen,hu2019difftaichi,huang2021plasticinelab,murthy2020gradsim} enable gradient-based inverse problems; in particular, the MPM-based morphing of Xu~et~al.~\cite{xu2025differentiable} provides physically plausible trajectories for large deformations.
These methods target grid-level losses; what trade-offs emerge when a point-level Chamfer objective is introduced remains unexplored.
However, mesh-based~\cite{sorkine2007rigid,eisenberger2021neuromorph}, implicit~\cite{park2019deepsdf,mescheder2019occupancy}, and generative~\cite{achlioptas2018learning,yang2019pointflow} approaches operate on fundamentally different representations and do not produce volumetric particle trajectories, which prevents direct numerical comparison with our particle-based setting.

\section{Analysis: Why Chamfer Optimization Fails}
\label{sec:analysis}

\subsection{Experimental Setup}
\label{sec:setup}

As a controlled testbed, we use an isosphere with approximately 89K particles as the source shape and a surface-sampled Stanford Bunny with approximately 74K points as the target.
The evaluation metric is two-sided Chamfer distance, defined as $\mathrm{CD} = \sqrt{\mathrm{mean}(d_{\text{fwd}}^2) + \mathrm{mean}(d_{\text{rev}}^2)}$ where $d_{\text{fwd}}$ and $d_{\text{rev}}$ are the per-point nearest-neighbor distances in the source-to-target (s$\to$t) and target-to-source (t$\to$s) directions, computed via exact nearest-neighbor search.
DCO is trained with AMSGrad (learning rate 0.01) for 40 frames of 30 steps each, and the physics-only baseline runs the same differentiable MPM simulation~\cite{xu2025differentiable} with the same particle count and frame budget. All methods share identical evaluation protocols and initial conditions.

\subsection{Direct Chamfer Optimization and the Collapse Mechanism}
\label{sec:dco_collapse}

DCO directly optimizes particle positions with respect to Chamfer distance via gradient descent, using no physics simulation, no morphing framework, and no temporal structure.
It therefore isolates the Chamfer gradient in its purest form, so the resulting failure patterns are due to the objective's gradient structure, not any domain-specific design choice.
As the first row of Table~\ref{tab:regularization} shows, DCO records a two-sided CD of 0.286, roughly 1.3$\times$ larger than the physics-only baseline (0.217), despite optimizing the very metric used for evaluation.
Both components plateau well above the physics baseline (s$\to$t: 0.22 vs.\ 0.13; t$\to$s: 0.18 vs.\ 0.17), reflecting a degenerate equilibrium in which source points cluster at shared target locations, producing many-to-one overlap that the two-sided objective cannot resolve.

The root cause lies in the structure of the Chamfer gradient itself.
Up to a positive scalar factor, the per-point contribution of the forward (s$\to$t) term is directed as
\begin{equation}
\frac{\partial \mathrm{CD}_{\text{s}\to\text{t}}}{\partial p} \;\propto\; p - \mathrm{NN}_{\text{target}}(p),
\label{eq:chamfer_grad}
\end{equation}
which points toward the nearest target point and carries no information about other source points or global coverage.
When multiple source points share the same nearest target, this gradient pulls all of them toward the same location.
In Appendix~\ref{sec:proofs}, we formalize this observation through three propositions and a corollary.
Proposition~1 shows that many-to-one collapse is the unique attractor of the forward Chamfer gradient.
Specifically, within a Voronoi cell where $k$ source points share the same nearest target $q^*$, the per-point gradient reduces to $\tfrac{2}{N}(p_i - q^*)$, whose unique zero is $p_i = q^*$ for all $i$; the joint Hessian $\tfrac{2}{N}I_{3k\times 3k}$ is positive definite, confirming the collapse equilibrium is a stable minimum (full proof in Appendix~\ref{sec:proof_prop1}).
Proposition~2 shows that the reverse (t$\to$s) gradient is nonzero for at most one of $k$ collapsed points, leaving the remaining $k{-}1$ with zero gradient and no mechanism for separation.
Proposition~3 proves that any local regularizer (repulsion, smoothness, volume preservation) whose gradient depends only on a particle's neighborhood preserves the net cluster-level drift toward the target: pairwise internal forces cancel, so the cluster centroid remains governed solely by the forward term regardless of regularization strength.
Corollary~1 establishes that collapse suppression requires coupling that propagates beyond local neighborhoods.
These results predict the empirical patterns reported below.

\subsection{Insufficiency of Local Regularizers}
\label{sec:reg_failure}

The most natural attempt to mitigate collapse is to augment the Chamfer objective with local regularizers.
Proposition~3 predicts that such regularizers cannot alter the drift at the cluster level; we now verify this empirically.
Table~\ref{tab:regularization} compares five conditions in which repulsion, smoothness, and volume preservation are added incrementally to DCO, alongside the physics-only baseline.

\begin{table}[t]
\centering
\scriptsize
\caption{Effect of adding local regularizers and density-aware re-weighting (DCD~\cite{wu2021density}) to DCO.  No condition improves t$\to$s coverage over pure DCO; the physics-only baseline outperforms all conditions.}
\label{tab:regularization}
\begin{tabular}{lcccc}
\toprule
Condition & s$\to$t & t$\to$s & Two-sided & vs DCO \\
\midrule
DCO (no reg.) & 0.220 & 0.182 & 0.286 & -- \\
+ Repulsion ($\lambda$=0.01) & 0.232 & 0.182 & 0.295 & worse \\
+ Repulsion ($\lambda$=0.1) & 0.250 & 0.182 & 0.309 & worse \\
+ Rep.\ + Smooth & 0.240 & 0.181 & 0.301 & worse \\
+ Rep.\ + Smooth + Vol. & 0.240 & 0.194 & 0.309 & worst \\
DCD~\cite{wu2021density} & 0.197 & 0.182 & 0.268 & $\sim$same \\
\midrule
Physics-only~\cite{xu2025differentiable} & \textbf{0.130} & \textbf{0.174} & \textbf{0.217} & 1.32$\times$ better \\
\bottomrule
\end{tabular}
\end{table}

Two patterns are evident.
First, regardless of the type or strength of regularization, the t$\to$s component remains nearly unchanged at approximately 0.181--0.194, indicating that coverage failure persists intact.
Second, adding regularizers actually increases the two-sided CD: the fully regularized setting yields the worst CD among all conditions.
Density-aware Chamfer Distance (DCD)~\cite{wu2021density}, which explicitly re-weights each point's contribution by the inverse of its local density, also fails to improve coverage: its t$\to$s (0.182) is identical to baseline DCO.
This confirms that the collapse is not a density-estimation problem that per-point re-weighting can fix, but a structural property of the Chamfer gradient direction itself, exactly as predicted by Proposition~3.

\subsection{Generalization Across Target Shapes}
\label{sec:shape_generalization}

We verify that these patterns are not specific to a single shape pair by repeating the experiment with the Sphere as source and four targets of varying geometric complexity.

\begin{table}[t]
\centering
\scriptsize
\caption{Generalization across four target shapes (Sphere source).  DCO and DCD produce nearly identical results on every shape; the t$\to$s component is virtually unchanged by density re-weighting, confirming that the coverage failure is structural.  The topologically complex dragon is analyzed separately in Section~\ref{sec:dragon}.}
\label{tab:shape_generalization}
\begin{tabular}{lccccc}
\toprule
Target & DCO & DCD~\cite{wu2021density} & Physics & DCO/Phys & DCD/Phys \\
\midrule
Bunny   & 0.286 & 0.268 & \textbf{0.217} & 1.3$\times$ & 1.2$\times$ \\
Duck     & 0.744 & 0.740 & \textbf{0.247} & 3.0$\times$ & 3.0$\times$ \\
Cow    & 0.449 & 0.436 & \textbf{0.228} & 2.0$\times$ & 1.9$\times$ \\
Teapot  & 0.209 & 0.199 & 0.230          & 0.9$\times$ & 0.9$\times$ \\
\bottomrule
\end{tabular}
\end{table}

As shown in Table~\ref{tab:shape_generalization}, DCO and DCD produce 1.3--3.0$\times$ worse two-sided CD than the physics-only baseline on geometrically complex targets, while only the nearly convex teapot falls within comparable range.
The t$\to$s coverage component is virtually identical between DCO and DCD on every shape, providing direct evidence that density re-weighting does not alter the many-to-one gradient structure.
The severity of collapse scales with target complexity, but the structural insensitivity to density-aware correction is universal.
Evaluating all 20 directed pairs among five shapes (Section~\ref{sec:exp_pairwise}) confirms that collapse is not an artifact of convex initialization: DCO produces worse two-sided CD than physics-only on the majority of pairs regardless of source geometry.
We defer the most topologically complex target (dragon, 2.1$\times$ collapse ratio) to Section~\ref{sec:dragon}.

\subsection{Cross-Domain Validation: 2D Collapse Experiment}
\label{sec:2d_experiment}

The preceding analysis and 3D experiments use shape morphing as a testbed.
To confirm that collapse is intrinsic to the Chamfer objective rather than specific to 3D morphing, we conduct a controlled 2D experiment: 600 points on a circle boundary are optimized to match a 200-point star target (Figure~\ref{fig:2d_experiment}).

\begin{figure}[t]
\centering
\includegraphics[width=\linewidth]{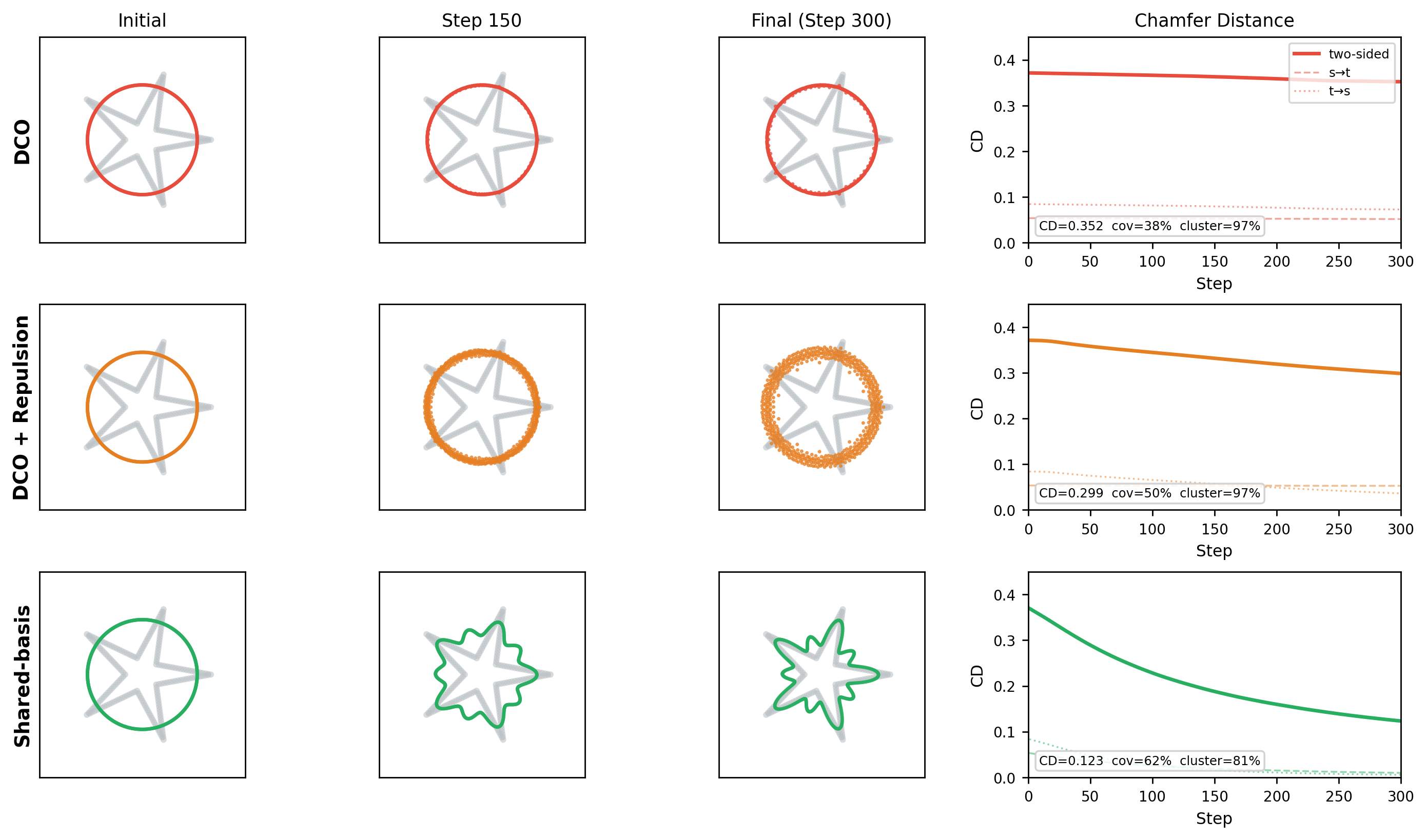}
\caption{\textbf{2D collapse experiment (circle$\to$star).}
\textbf{Top}: Per-point CD optimization (DCO) collapses source points onto star vertices.
\textbf{Middle}: Adding local repulsion does not resolve collapse (Proposition~3).
\textbf{Bottom}: Shared-basis (Fourier) deformation provides global coupling that suppresses collapse.
Gray: target; colored: source at initial / mid / final steps.  Right column: CD convergence.}
\label{fig:2d_experiment}
\end{figure}

We compare three strategies:
(i)~\textbf{DCO}: per-point CD gradient descent collapses onto star vertices (coverage 38\%, cluster fraction 97\%), exactly as Proposition~1 predicts.
(ii)~\textbf{DCO + Repulsion}: a local $k$-NN repulsive force ($1/d^2$) leaves cluster fraction at 97\%, consistent with Proposition~3.
(iii)~\textbf{Shared-basis deformation}: the boundary is parameterized as $r(\theta) = a_0 + \sum_{k}[a_k\cos k\theta + b_k\sin k\theta]$ with 12 Fourier modes; all points share the same coefficients.
CD gradients backpropagate through this parameterization, coupling all points globally, analogous to MPM's shared Eulerian grid.
The result: 2.9$\times$ lower CD (0.123 vs.\ 0.352) and reduced clustering (81\% vs.\ 97\%).
This confirms that (a)~collapse is domain-agnostic, and (b)~the design principle of Corollary~1 holds: any form of global coupling, not only MPM, suppresses collapse.
Full details are provided in Appendix~\ref{sec:2d_collapse}.

\section{Validating Global Coupling via Physics-Guided Optimization}
\label{sec:method}

The analysis in Section~\ref{sec:analysis} establishes that collapse suppression requires coupling beyond local neighborhoods.
To test this, we realize global coupling through the differentiable MPM framework of Xu~et~al.~\cite{xu2025differentiable}, built on the material point method~\cite{sulsky1994particle,stomakhin2013material,hu2018moving}.
In MPM, all particles are coupled through a shared Eulerian grid: moving a single particle generates elastic stress that propagates through the continuum, providing the global interaction that local regularizers lack.
We deliberately keep the integration simple to isolate the effect of global coupling itself, rather than introducing architectural novelty.
Simulation details and material parameters are provided in Appendix~\ref{sec:implementation}.

\paragraph{The Chamfer gap.}
The physics prior optimizes a grid-level mass-density loss that measures the discrepancy between the current and target mass distributions on the Eulerian grid:
\begin{equation}
\mathcal{L}_{\text{physics}} = \sum_{i} \tfrac{1}{2}\bigl(\ln(m_i + 1 + \epsilon) - \ln(m^{*}_i + 1 + \epsilon)\bigr)^{2},
\label{eq:physics_loss}
\end{equation}
where $m_i$ and $m^{*}_i$ are the current and target grid node masses, and the log transform ensures scale-invariant matching across regions of varying density.
This loss guarantees collapse-free trajectories, but operates on voxelized mass distributions rather than individual particle positions, leaving a residual Chamfer gap with respect to point-level distance metrics.
The core challenge is to reduce this gap without reintroducing collapse.

\paragraph{Joint optimization.}
We integrate a bidirectional Chamfer loss directly into the MPM computational graph. Given source particles $\{p_i\}_{i=1}^{N}$ at the final timestep and target points $\{q_j\}_{j=1}^{M}$, the forward term $\mathcal{L}_{\text{fwd}}$ is the standard source-to-target Chamfer loss whose per-point gradient was analyzed in Eq.~\eqref{eq:chamfer_grad}. The reverse term pulls each uncovered target point toward its nearest source particle:
\begin{equation}
\mathcal{L}_{\text{rev}} = \frac{1}{M}\sum_{j=1}^{M}\|q_j - \mathrm{NN}_{\mathcal{S}}(q_j)\|^2.
\label{eq:rev_chamfer}
\end{equation}
The total loss combines the physics prior with this bidirectional Chamfer objective, with all weights varying as a function of the current frame $t$:
\begin{equation}
\begin{gathered}
\mathcal{L}_{\text{total}} = w_{\text{phys}}(t)\,\mathcal{L}_{\text{physics}}
  + w_{\text{ch}}(t)\,\mathcal{L}_{\text{chamfer}}, \\
\mathcal{L}_{\text{chamfer}} = \mathcal{L}_{\text{fwd}}
  + w_{\text{rev}}(t)\,\mathcal{L}_{\text{rev}}.
\end{gathered}
\label{eq:joint_loss}
\end{equation}
A na\"ive joint optimization fails because the reverse term (t$\to$s) introduces strong shrinkage pressure: when the physics weight decreases, elastic resistance weakens while reverse pressure remains high, causing transient contraction.
We introduce a coupled schedule that ties the reverse weight to the physics weight:
\begin{equation}
w_{\text{rev}}(t) = w_{\text{rev,base}} \cdot w_{\text{phys}}(t).
\label{eq:coupled_schedule}
\end{equation}
When physics is strong, full reverse pressure maintains coverage; as physics weakens, reverse pressure automatically decreases.
This directly follows from the analysis: physics coupling provides the counter-force that offsets the many-to-one collapse pressure.

\paragraph{Reverse gradient clamping.}
The reverse Chamfer gradient concentrates on a small subset of boundary particles, producing sparse, high-magnitude updates.
We clamp outlier magnitudes exceeding $\kappa$ times the mean active magnitude ($\kappa=3$) to prevent destabilizing the physics solver.
The forward (s$\to$t) gradient needs no clamping as every source particle receives exactly one gradient vector.

\paragraph{System overview.}
Starting from a source mesh, we sample MPM particles and optimize the per-step deformation field for 40 frames using the joint loss of Eq.~\eqref{eq:joint_loss} with the coupled schedule of Eq.~\eqref{eq:coupled_schedule}.
The schedule transitions from physics-dominated to Chamfer-guided refinement, with the coupled reverse weight ensuring that shrinkage pressure never exceeds elastic resistance.

\section{Experiments}
\label{sec:experiments}

\subsection{Setup}
\label{sec:exp_setup}

We evaluate on four target shapes (Bunny, Duck, Cow, Teapot) plus a separate dragon case study (Section~\ref{sec:dragon}).
All experiments use a $32^3$ MPM grid with 3 particles per cell (${\sim}$37K particles), neo-Hookean elasticity, 10 timesteps per frame, and 40 frames.
We compare physics-only ($w_{\text{chamfer}}{=}0$) against our joint optimization ($w_{\text{chamfer}}{=}10$, linear ramp from frame~5 to~15, coupled reverse schedule, $\kappa{=}3$ clamping, physics decay after frame~20).
No existing volumetric morphing method optimizes Chamfer distance on particle systems; DCO and DCD isolate the Chamfer gradient in its purest form, providing the most controlled comparison for validating our analysis.

\subsection{Results: Sphere Source}
\label{sec:exp_results}

\begin{table}[t]
\centering
\scriptsize
\caption{Chamfer distance at frame~39 for physics-only and our method (Sphere source).  Our method improves s$\to$t on all four shapes (average $-$10.8\%) and two-sided CD on all four.}
\label{tab:coupled_results}
\begin{tabular}{lcccccc}
\toprule
 & \multicolumn{2}{c}{s$\to$t $\downarrow$} & \multicolumn{2}{c}{t$\to$s $\downarrow$} & \multicolumn{2}{c}{Two-sided $\downarrow$} \\
\cmidrule(lr){2-3}\cmidrule(lr){4-5}\cmidrule(lr){6-7}
Target & Physics & Ours & Physics & Ours & Physics & Ours \\
\midrule
Bunny   & 0.181 & \textbf{0.157} & 0.174 & \textbf{0.163} & 0.251 & \textbf{0.226} \\
Duck     & 0.184 & \textbf{0.173} & \textbf{0.214} & 0.215 & 0.282 & \textbf{0.276} \\
Cow    & 0.178 & \textbf{0.167} & 0.185 & \textbf{0.167} & 0.257 & \textbf{0.236} \\
Teapot  & 0.180 & \textbf{0.148} & 0.190 & \textbf{0.181} & 0.262 & \textbf{0.234} \\
\bottomrule
\end{tabular}
\end{table}

Table~\ref{tab:coupled_results} reports results at the final frame.
The s$\to$t component improves on all four shapes (average $-$10.8\%), ranging from $-$6.0\% (Duck) to $-$17.8\% (teapot).
Two-sided CD improves on all four shapes, confirming that joint optimization reduces the geometric gap without inducing the severe collapse of DCO.
To verify that improvement is not metric-specific, Table~\ref{tab:hausdorff} reports Hausdorff distance and F1-score ($\tau{=}0.2$) for the same configurations.
Frame-wise convergence curves are provided in Appendix~\ref{sec:cd_convergence}; a trajectory comparison with DCD is shown in Figure~\ref{fig:trajectory}.

\begin{table}[t]
\centering
\scriptsize
\caption{Hausdorff distance ($\downarrow$) and F1-score ($\uparrow$, $\tau{=}0.2$) at frame~39 (Sphere source).  Improvements are consistent across metrics, confirming that the Chamfer gap reduction reflects genuine geometric improvement.}
\label{tab:hausdorff}
\begin{tabular}{lcccc}
\toprule
 & \multicolumn{2}{c}{Hausdorff $\downarrow$} & \multicolumn{2}{c}{F1 ($\tau{=}0.2$) $\uparrow$} \\
\cmidrule(lr){2-3}\cmidrule(lr){4-5}
Target & Physics & Ours & Physics & Ours \\
\midrule
Bunny   & 0.963 & \textbf{0.580} & 0.757 & \textbf{0.851} \\
Duck     & 0.930 & \textbf{0.646} & 0.626 & \textbf{0.643} \\
Cow    & 0.715 & \textbf{0.672} & 0.733 & \textbf{0.819} \\
Teapot  & 0.954 & \textbf{0.842} & 0.688 & \textbf{0.810} \\
\bottomrule
\end{tabular}
\end{table}

\subsection{Results: Pairwise Morphing}
\label{sec:exp_pairwise}

\begin{table}[t]
\centering
\scriptsize
\caption{Pairwise morphing results (two-sided CD) in matrix format.  Each cell shows Physics\,/\,\textbf{Ours}.  \textbf{Bold}: improvement over physics-only.  Our method improves two-sided CD on 16 of 20 directed pairs.}
\label{tab:pairwise}
\setlength{\tabcolsep}{3pt}
\begin{tabular}{l ccccc}
\toprule
Src\,$\downarrow$\,/\,Tgt\,$\rightarrow$ & Sphere & Bunny & Duck & Cow & Teapot \\
\midrule
Sphere & ---                      & 0.251\,/\,\textbf{0.226} & 0.282\,/\,\textbf{0.276} & 0.257\,/\,\textbf{0.236} & 0.262\,/\,\textbf{0.234} \\
Bunny  & 0.295\,/\,\textbf{0.250} & --- & 0.297\,/\,\textbf{0.288} & 0.256\,/\,\textbf{0.240} & 0.286\,/\,\textbf{0.259} \\
Duck    & \textbf{0.293}\,/\,0.296 & 0.283\,/\,\textbf{0.276} & --- & \textbf{0.280}\,/\,0.295 & 0.326\,/\,\textbf{0.309} \\
Cow   & \textbf{0.291}\,/\,0.315 & 0.275\,/\,\textbf{0.252} & 0.425\,/\,\textbf{0.284} & --- & 0.318\,/\,\textbf{0.293} \\
Teapot & 0.226\,/\,\textbf{0.213} & 0.257\,/\,\textbf{0.252} & 0.355\,/\,\textbf{0.312} & \textbf{0.237}\,/\,0.259 & --- \\
\bottomrule
\end{tabular}
\end{table}

Table~\ref{tab:pairwise} reports two-sided CD across all 20 directed pairs.
Our method improves s$\to$t on all 20 pairs (see Appendix~\ref{sec:pairwise_s2t} for per-pair breakdown) and two-sided CD on 16 of 20.
The four regressions all target the Sphere or Cow, where the physics baseline already achieves near-optimal coverage.

\subsection{Case Study: Topologically Complex Target (Dragon)}
\label{sec:dragon}

The Stanford dragon presents the most challenging target in our benchmark: its deep concavities, thin appendages, and high genus produce the most severe many-to-one collapse under DCO (2.1$\times$ worse than physics-only).
We isolate this case to examine how particle resolution interacts with global coupling on topologically extreme geometry.

\begin{table}[t]
\centering
\scriptsize
\caption{Sphere$\to$dragon results. At the default resolution of 3 particles per cell, our method improves s$\to$t but worsens t$\to$s due to insufficient particle count. Increasing to 4 particles per cell resolves this limitation, achieving a \textbf{2.5$\times$} improvement in two-sided CD over the physics-only baseline.  Hausdorff distance and F1-score confirm the same trend.}
\label{tab:dragon}
\begin{tabular}{lccccc}
\toprule
Method & s$\to$t $\downarrow$ & t$\to$s $\downarrow$ & Two-sided $\downarrow$ & Hausdorff $\downarrow$ & F1 $\uparrow$ \\
\midrule
DCO              & 0.609 & 0.255 & 0.660 & 2.695 & 0.466 \\
DCD~\cite{wu2021density} & 0.608 & 0.255 & 0.659 & 2.849 & 0.463 \\
Physics-only     & 0.190 & \textbf{0.246} & 0.311 & 3.202 & 0.532 \\
Ours (3 particles per cell)   & \textbf{0.160} & 0.296 & 0.336 & 3.251 & 0.493 \\
Ours (4 particles per cell)   & \textbf{0.089} & \textbf{0.087} & \textbf{0.124} & \textbf{0.780} & \textbf{0.792} \\
\bottomrule
\end{tabular}
\end{table}

\begin{figure}[t!]
    \centering
    \includegraphics[width=\linewidth]{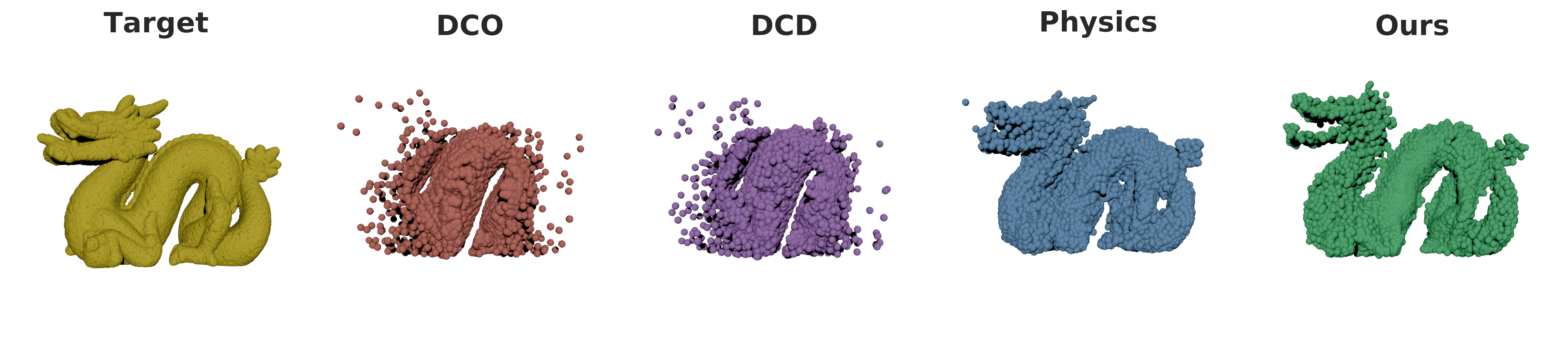}
    \caption{\textbf{Sphere$\to$dragon: qualitative comparison.}
    DCO and DCD produce severe surface collapse.
    Physics-only preserves volumetric structure but leaves a large geometric gap.
    Our method at 4 particles per cell closely matches the target geometry while maintaining physical validity.}
    \label{fig:dragon}
\end{figure}

Table~\ref{tab:dragon} reports the full progression.
At 3 particles per cell (${\sim}$37K particles), our method improves s$\to$t by 15.8\% but the conservative particle count cannot resolve the dragon's fine geometric detail, causing t$\to$s to increase.
Note that DCO achieves a lower Hausdorff distance than physics-only (2.695 vs.\ 3.202) because collapsed points sit directly on target surfaces, minimizing worst-case distance; yet its F1-score is the lowest among all methods (0.466), confirming that low Hausdorff under collapse is a measurement artifact of poor coverage, not genuine geometric quality.
Increasing to 4 particles per cell (${\sim}$89K particles) with a C++ inline Chamfer gradient implementation eliminates this trade-off: the two-sided CD drops to \textbf{0.124}, a \textbf{2.5$\times$ improvement} over the physics-only baseline, with Hausdorff and F1 confirming the same trend (0.780 and 0.792, respectively).
Figure~\ref{fig:dragon} shows the qualitative comparison.
We adopt 3 particles per cell as the default for the main pairwise evaluation because increasing to 4 particles per cell causes spray-like artifacts and occasional NaN divergence on simpler shapes; the dragon result demonstrates that per-shape resolution tuning can unlock substantially better performance on complex targets.

\begin{figure}[t!]
    \centering
    \includegraphics[width=\linewidth]{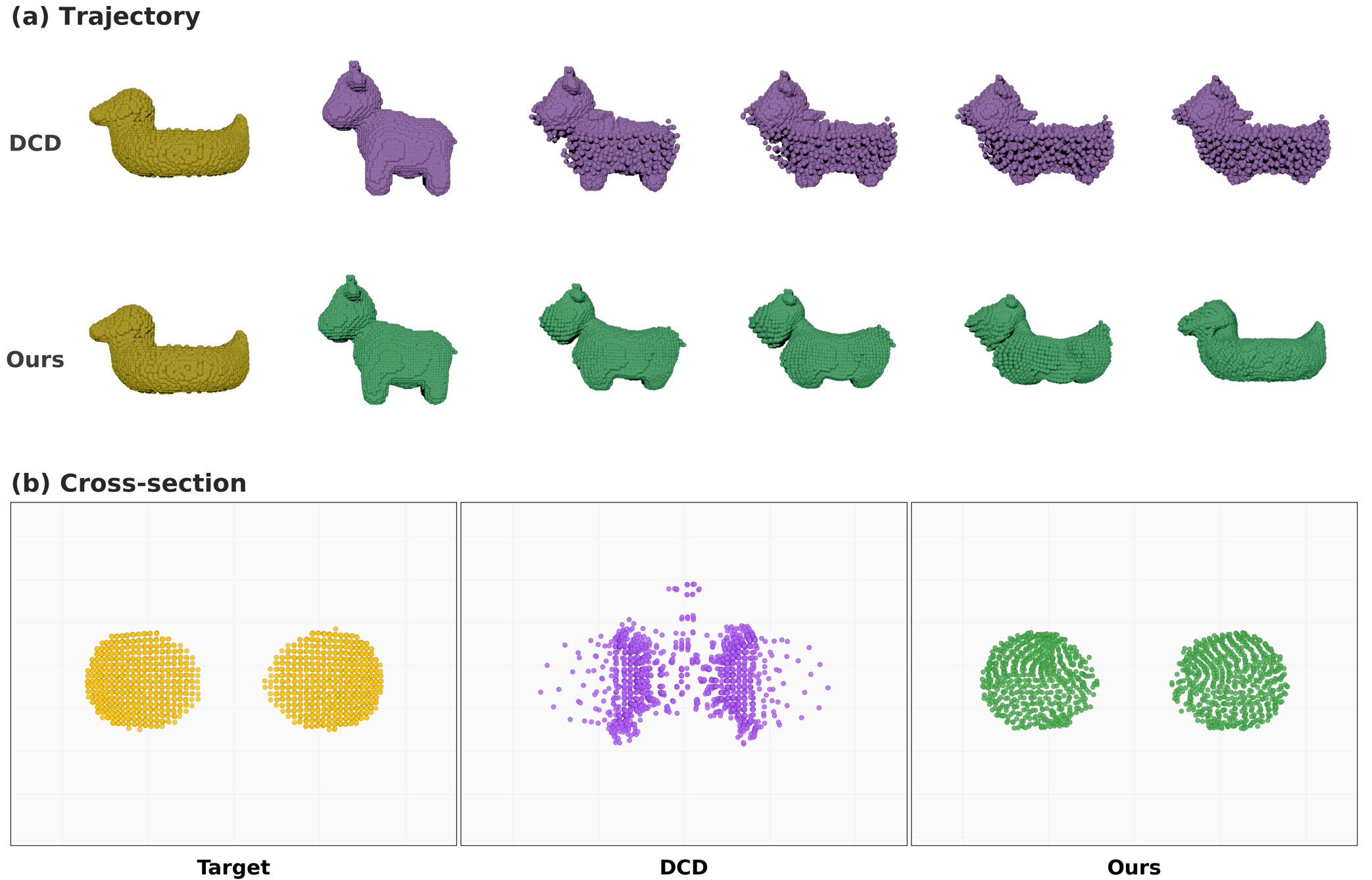}
    \caption{\textbf{DCD vs.\ Ours on Cow$\to$Duck: trajectory and internal structure.}
    Top two rows: morphing trajectory showing the target (yellow, leftmost).
    DCD dissolves the source structure into a noisy surface approximation;
    our method maintains volumetric coherence throughout.
    Bottom: cross-section at the mid-body plane.
    The target and our method show uniformly filled interiors, while DCD collapses particles onto the surface, leaving the interior hollow.
    16.5\% of DCD particles have nearest-neighbor distance $<$0.01 (effectively overlapping), compared to 0\% for our method.}
    \label{fig:trajectory}
\end{figure}

Figure~\ref{fig:trajectory} provides qualitative evidence of how collapse manifests in practice: DCD's density re-weighting cannot prevent interior hollowing, whereas our method's global coupling preserves volumetric integrity throughout the trajectory.

\section{Discussion}
\label{sec:discussion}

\paragraph{Generality and global coupling.}
The theoretical results (Propositions~1--3, Corollary~1) hold for arbitrary point sets optimized against a nearest-neighbor-based objective, independent of any morphing framework or physics prior.
The gradient attractor (Proposition~1) and the inability of local remedies to alter it (Proposition~3) therefore apply equally to point cloud generation~\cite{achlioptas2018learning,yang2019pointflow}, shape completion~\cite{yuan2018pcn,xie2020grnet}, and neural implicit fitting~\cite{gropp2020implicit}.
For example, Proposition~1 predicts that any CD-trained generative model will produce many-to-one clustering at high-density target regions; the ``mode collapse'' patterns reported in point cloud GANs~\cite{achlioptas2018learning} and the non-uniform sampling artifacts in flow-based generators~\cite{yang2019pointflow} are consistent with this prediction.
Similarly, CD-supervised shape completion networks~\cite{yuan2018pcn,xie2020grnet} are known to produce blurred outputs near thin structures, exactly the geometry where Voronoi cells are large and the collapse attractor is strongest.
Proposition~2 explains why switching from unidirectional to bidirectional CD does not resolve the issue: the reverse gradient can separate at most one of $k$ collapsed points, leaving the remaining $k{-}1$ in a zero-gradient deadlock, consistent with the limited diversity improvements reported when adding the reverse term in generative settings~\cite{achlioptas2018learning}.
Proposition~3 further explains why popular remedies such as repulsion penalties and density-aware re-weighting (DCD)~\cite{wu2021density} fail to eliminate collapse: any local pairwise regularizer cancels at the cluster centroid, leaving the net drift toward the target unchanged.
These are not architecture-specific failures but consequences of the NN-based objective's gradient structure.
MPM provides the necessary global coupling (Corollary~1) via continuum stress propagation; the analysis identifies the required property of non-local interaction without prescribing a particular implementation.
In other domains, non-local coupling could arise from, \eg, graph-based message passing or global latent variables.
The 2D cross-domain experiment (Section~\ref{sec:2d_experiment}) confirms that this principle generalizes: shared-basis deformation, an entirely different form of global coupling, suppresses collapse in 2D just as MPM does in 3D.

\paragraph{Stability--accuracy trade-off.}
Increasing from 3 to 4 particles per cell achieves a 2.5$\times$ improvement on the dragon (Section~\ref{sec:dragon}) but causes spray artifacts on simpler shapes.
Extending the stable operating regime to higher resolutions across all shape categories remains the primary open challenge.

\paragraph{Practical guidelines.}
\label{sec:guidelines}
Our analysis enables practical recommendations based on the DCO/Physics ratio in Table~\ref{tab:shape_generalization}:
\emph{Low risk} (nearly convex targets, DCO/Phys $\leq 1\times$): CD can be used alone.
\emph{Moderate risk} (1.3--2.0$\times$, \eg, bunny, Cow): a global coupling mechanism is recommended.
\emph{High risk} ($>2\times$, \eg, Duck; see also the dragon case study in Section~\ref{sec:dragon}): CD without global coupling is strongly discouraged.

\section{Limitations and Future Work}
\label{sec:limitations}

Our analysis is general: the 2D experiment (Section~\ref{sec:2d_experiment}) confirms that collapse and its suppression via global coupling are domain-agnostic, but 3D validation currently focuses on shape morphing.
Identifying domain-appropriate forms of global coupling for point cloud generation (e.g., graph-based message passing), shape completion (e.g., implicit field regularization), and neural implicit fitting is a natural next step; the design criterion established by Corollary~1 provides a clear starting point for such extensions.
On the simulation side, all experiments rely on neo-Hookean elasticity on a $32^{3}$ grid; heterogeneous materials and higher resolutions may alter the physics-only baseline.
Applying our analysis to EMD~\cite{rubner2000earth} and sliced Wasserstein~\cite{nguyen2020distributional,cuturi2013sinkhorn} would clarify whether the collapse is specific to nearest-neighbor-based objectives or common to point-level metrics more broadly.

\section{Conclusion}
\label{sec:conclusion}

We have identified a structural failure mode of Chamfer distance optimization: the many-to-one gradient structure drives optimization toward collapse-prone equilibria that no local remedy can resolve.
Through formal analysis (Propositions~1--3; proofs in Appendix~\ref{sec:proofs}), we derive a necessary condition for collapse suppression: coupling must propagate beyond local neighborhoods (Corollary~1).
This design principle is domain-agnostic, as confirmed by a controlled 2D experiment in which shared-basis deformation suppresses collapse that per-point optimization cannot.
In 3D shape morphing, realizing global coupling through a differentiable MPM prior validates the principle across 20 directed shape pairs, with a 2.5$\times$ improvement on the topologically complex dragon at higher particle resolution (Section~\ref{sec:dragon}).
Chamfer distance remains a valid evaluation metric; the failure we identify is specific to its use as an optimization objective, where the gradient structure, not the distance value, governs convergence.
Our analysis provides a clear design criterion for any pipeline that optimizes point-level distance metrics: when CD is used as a loss, the architecture must provide non-local coupling to counteract the collapse attractor; metric-level redesign alone cannot.

\bibliographystyle{splncs04}
\bibliography{main}

\appendix
\section*{Appendix}

\section{Formal Proofs}
\label{sec:proofs}

We provide complete proofs of Propositions~1--3 and Corollary~1 stated in the main paper.

\paragraph{Setup.}
Let $\mathcal{S} = \{p_1, \ldots, p_N\} \subset \mathbb{R}^3$ denote source points and $\mathcal{T} = \{q_1, \ldots, q_M\} \subset \mathbb{R}^3$ denote fixed target points.
The forward (s$\to$t) Chamfer loss is
\begin{equation}
\mathcal{L}_{\text{fwd}} = \frac{1}{N}\sum_{i=1}^{N}\|p_i - \mathrm{NN}_{\mathcal{T}}(p_i)\|^2,
\label{eq:fwd_loss}
\end{equation}
where $\mathrm{NN}_{\mathcal{T}}(p) = \arg\min_{q \in \mathcal{T}} \|p - q\|$.
The reverse (t$\to$s) Chamfer loss is
\begin{equation}
\mathcal{L}_{\text{rev\_app}} = \frac{1}{M}\sum_{j=1}^{M}\|q_j - \mathrm{NN}_{\mathcal{S}}(q_j)\|^2.
\label{eq:rev_loss_app}
\end{equation}
All analysis is conducted within a fixed region of nearest-neighbor assignment (\ie, within a single Voronoi cell where $\mathrm{NN}_{\mathcal{T}}(p_i)$ does not change).

\subsection{Proposition 1: Many-to-One Collapse Is the Unique Attractor}
\label{sec:proof_prop1}

\begin{proposition}
Consider a cluster $C = \{p_1, \ldots, p_k\}$ of source points sharing the same nearest target $q^* = \mathrm{NN}_{\mathcal{T}}(p_i)$ for all $i$.
Under the forward Chamfer gradient:
\begin{enumerate}
    \item[(a)] The state $p_i = q^*$ for all $i$ is the unique equilibrium of $C$.
    \item[(b)] The Hessian at this equilibrium is positive definite, making it a stable attractor.
\end{enumerate}
\end{proposition}

\begin{proof}
\textbf{Part (a).}
Within the Voronoi cell where $\mathrm{NN}_{\mathcal{T}}(p_i) = q^*$, the gradient of the forward loss with respect to $p_i$ is
\begin{equation}
\frac{\partial \mathcal{L}_{\text{fwd}}}{\partial p_i} = \frac{2}{N}(p_i - q^*).
\end{equation}
Setting this to zero yields $p_i = q^*$.
Since $q^*$ is fixed, this solution is unique for each $i \in \{1, \ldots, k\}$.
Hence the unique equilibrium is the many-to-one collapse state $p_1 = p_2 = \cdots = p_k = q^*$.

\textbf{Part (b).}
The Hessian of the forward loss with respect to $p_i$ is
\begin{equation}
\frac{\partial^2 \mathcal{L}_{\text{fwd}}}{\partial p_i^2} = \frac{2}{N}\, I_{3 \times 3},
\end{equation}
which is positive definite.
For the full cluster, the joint Hessian is $\frac{2}{N}\, I_{3k \times 3k}$ (block-diagonal with identical blocks), which is positive definite.
Therefore the collapse equilibrium is a stable local minimum and the unique attractor within this Voronoi cell.
\end{proof}

\begin{remark}
Globally, source space is partitioned into Voronoi cells of the target points.
Within each cell, Proposition~1 guarantees that all source points collapse onto the corresponding target point.
The global attractor is thus the many-to-one collapse state where each source point sits exactly on its nearest target.
Note that during gradient descent a source point may cross a Voronoi boundary and be reassigned to a different target; however, this merely redirects the point toward a \emph{new} attractor of the same kind, so the many-to-one collapse structure is preserved globally.
\end{remark}

\subsection{Proposition 2: The Reverse Term Cannot Separate Collapsed Points}
\label{sec:proof_prop2}

\begin{proposition}
Let $k$ source points be collapsed at a target point: $p_1 = p_2 = \cdots = p_k = q^*$.
The reverse Chamfer gradient $\partial \mathcal{L}_{\emph{rev}} / \partial p_i$ is nonzero for at most one of the $k$ points; the remaining $k{-}1$ receive zero gradient.
\end{proposition}

\begin{proof}
The reverse loss sums over target points:
\begin{equation}
\mathcal{L}_{\text{rev}} = \frac{1}{M}\sum_{j=1}^{M}\|q_j - \mathrm{NN}_{\mathcal{S}}(q_j)\|^2.
\end{equation}
Source point $p_i$ contributes to $\mathcal{L}_{\text{rev}}$ only if $\mathrm{NN}_{\mathcal{S}}(q_j) = p_i$ for some target point $q_j$.
When $p_1 = \cdots = p_k = q^*$, all $k$ points occupy the same spatial location.
For any target point $q_j$ whose nearest source is at $q^*$, the nearest-neighbor search returns exactly one index $i^*$ via tie-breaking (typically the one with the smallest index).
Therefore:
\begin{equation}
\frac{\partial \mathcal{L}_{\text{rev}}}{\partial p_i} =
\begin{cases}
\displaystyle -\frac{2}{M}\sum_{\substack{j:\, \mathrm{NN}_{\mathcal{S}}(q_j) = p_{i^*}}} (q_j - q^*) & \text{if } i = i^*, \\[8pt]
0 & \text{if } i \neq i^*.
\end{cases}
\end{equation}
At most one point ($i^*$) receives a nonzero gradient from the reverse term; the remaining $k{-}1$ points have zero reverse gradient.

Furthermore, at the collapse point the forward gradient is also zero (by Proposition~1), so these $k{-}1$ points have zero total gradient under the bidirectional objective.
They are stuck at $q^*$ with no mechanism for separation.
\end{proof}

\begin{remark}
This explains why bidirectional Chamfer distance cannot resolve collapse: the forward term is at equilibrium (Proposition~1), and the reverse term provides a separation signal to at most one of the $k$ collapsed points.
The remaining $k{-}1$ points are in a deadlock state with zero gradient from both terms.
\end{remark}

\subsection{Proposition 3: Local Regularizers Cannot Alter Cluster-Level Drift}
\label{sec:proof_prop3}

\begin{proposition}
Let $C = \{p_1, \ldots, p_k\}$ be a cluster of source points sharing the same nearest target $q^* = \mathrm{NN}_{\mathcal{T}}(p_i)$.
Let $R$ be any local regularizer that is translationally invariant, i.e., $R(\{p_i + \mathbf{v}\}) = R(\{p_i\})$ for any $\mathbf{v} \in \mathbb{R}^3$.
Then the dynamics of the cluster centroid $\bar{p} = \frac{1}{k}\sum_{i} p_i$ under the combined loss $\mathcal{L}_{\emph{fwd}} + \lambda R$ are independent of $\lambda$:
\begin{equation}
\frac{d\bar{p}}{dt} = -\frac{2\eta}{N}(\bar{p} - q^*),
\end{equation}
where $\eta$ is the learning rate.
\end{proposition}

\begin{proof}
Under gradient descent, the centroid evolves as
\begin{equation}
\frac{d\bar{p}}{dt} = -\frac{\eta}{k}\sum_{i=1}^{k}\left(\frac{\partial \mathcal{L}_{\text{fwd}}}{\partial p_i} + \lambda \frac{\partial R}{\partial p_i}\right).
\end{equation}
We evaluate each sum separately.

\textbf{Forward term.}
\begin{equation}
\sum_{i=1}^{k} \frac{\partial \mathcal{L}_{\text{fwd}}}{\partial p_i} = \sum_{i=1}^{k} \frac{2}{N}(p_i - q^*) = \frac{2}{N}\left(\sum_{i=1}^{k} p_i - k\, q^*\right) = \frac{2k}{N}(\bar{p} - q^*).
\end{equation}

\textbf{Regularizer term.}
Since $R$ is translationally invariant, differentiating with respect to the translation parameter $\mathbf{v}$ at $\mathbf{v}=0$ gives $\sum_{i=1}^{k} \partial R / \partial p_i = 0$ directly.

\textbf{Combined.}
Substituting:
\begin{equation}
\frac{d\bar{p}}{dt} = -\frac{\eta}{k}\left(\frac{2k}{N}(\bar{p} - q^*) + \lambda \cdot 0\right) = -\frac{2\eta}{N}(\bar{p} - q^*).
\end{equation}
The centroid moves toward $q^*$ at a rate independent of $\lambda$.
Local regularizers may rearrange points within the cluster (\eg, push them apart via repulsion), but they cannot alter the cluster's net drift toward the target.
\end{proof}

\begin{remark}
Translational invariance is satisfied by all standard local regularizers used in point cloud optimization; pairwise antisymmetry (Newton's third law) is one sufficient condition for it:
\begin{itemize}
    \item \textbf{Repulsion}: $f(p_i, p_j) = \lambda_r (p_i - p_j) / \|p_i - p_j\|^{2+\alpha}$, which satisfies $f(p_i,p_j) = -f(p_j,p_i)$.
    \item \textbf{Laplacian smoothness}: The gradient $p_i - \frac{1}{|\mathcal{N}(i)|}\sum_j p_j$ decomposes into pairwise terms that cancel under summation over symmetric neighborhoods.
    \item \textbf{Volume preservation}: Local volume penalties produce gradients from pairwise particle interactions that obey Newton's third law.
    \item \textbf{Density-aware re-weighting (DCD)}: DCD reweights each forward contribution by $w_i = 1/\hat{\rho}(p_i)$, yielding a primary gradient $\frac{2w_i}{N}(p_i - \mathrm{NN}_{\mathcal{T}}(p_i))$ that preserves the direction toward the nearest target, plus a secondary term $\propto \|p_i - \mathrm{NN}_{\mathcal{T}}(p_i)\|^2\,\partial w_i/\partial p_i$ that vanishes at the collapse equilibrium $\|p_i - q^*\| = 0$.  Near collapse, DCD therefore leaves the attraction structure invariant, consistent with the identical t$\to$s values between DCO and DCD in Table~2.
\end{itemize}
\end{remark}

\subsection{Corollary 1: Collapse Suppression Requires Non-Local Coupling}
\label{sec:proof_cor1}

\begin{corollary}
Taken together, Propositions~1 and~3 imply that any coupling mechanism whose gradients depend only on particles within a local neighborhood $\mathcal{N}(i)$ cannot, in general, suppress many-to-one collapse while still reducing the Chamfer objective.
Collapse suppression therefore requires coupling that propagates beyond $\mathcal{N}(i)$.
\end{corollary}

\begin{proof}[Argument]
Suppose a coupling mechanism $G$ suppresses collapse, but its gradient is local:
\begin{equation}
\frac{\partial G}{\partial p_i} = g\bigl(\{p_j : j \in \mathcal{N}(i)\}\bigr).
\end{equation}

\textbf{Case 1 (pairwise local forces).}
If $G$ is translationally invariant, then by Proposition~3 its net contribution to the cluster centroid vanishes.
The centroid dynamics remain governed solely by the forward Chamfer term, and collapse proceeds as in the unregularized case.

\textbf{Case 2 (general local forces).}
Even if $G$ does not decompose into pairwise forces, a local mechanism faces a tension arising from the positive-definite Hessian of Proposition~1.
The forward Chamfer gradient defines a strictly convex basin around $q^*$; any perturbation away from $q^*$ is restored by the $\frac{2}{N}I$ curvature.
A local coupling term $G$ can oppose this basin only within its neighborhood radius.
If $\|{\partial G}/{\partial p_i}\| < \|{\partial \mathcal{L}_{\text{fwd}}}/{\partial p_i}\|$ near $q^*$, the Chamfer basin dominates and collapse proceeds.
If $\|{\partial G}/{\partial p_i}\| > \|{\partial \mathcal{L}_{\text{fwd}}}/{\partial p_i}\|$, the coupling overwhelms the Chamfer signal in that region, effectively replacing the Chamfer objective with the regularizer's own objective.
Table~2 of the main paper confirms this trade-off empirically: strong local repulsion ($\lambda = 0.1$) worsens two-sided CD from 0.286 to 0.309.

In neither case can purely local coupling suppress collapse while improving the Chamfer objective.
Non-local coupling circumvents this trade-off by providing a restoring force that does not originate from the immediate neighborhood of the collapsing cluster, and therefore does not oppose the per-point Chamfer gradient locally.
In MPM, this is realized through continuum stress propagation on the shared Eulerian grid: moving a single particle generates elastic stress that propagates throughout the domain, coupling distant particles through the continuum.
\end{proof}


\section{Implementation Details}
\label{sec:implementation}

\subsection{MPM Simulation Parameters}
\label{sec:sim_params}

All experiments use a differentiable Material Point Method (MPM) simulation based on the framework of Xu~et~al.~\cite{xu2025differentiable}.
Table~\ref{tab:sim_params} summarizes the simulation parameters shared across all experiments.

\begin{table}[h]
\centering
\caption{MPM simulation parameters. All experiments share the same material model and grid configuration; only particle resolution varies.}
\label{tab:sim_params}
\begin{tabular}{lcc}
\toprule
Parameter & Symbol & Value \\
\midrule
\multicolumn{3}{l}{\textit{Grid}} \\
\quad Grid spacing & $\Delta x$ & 1.0 \\
\quad Domain & $[\mathbf{x}_{\min}, \mathbf{x}_{\max}]$ & $[-16, 16]^3$ \\
\quad Effective resolution & & $32^3$ nodes \\
\midrule
\multicolumn{3}{l}{\textit{Material (Neo-Hookean)}} \\
\quad First Lam\'e parameter & $\lambda$ & 38\,889 \\
\quad Shear modulus & $\mu$ & 58\,333 \\
\quad Density & $\rho$ & 75.0 \\
\midrule
\multicolumn{3}{l}{\textit{Time integration}} \\
\quad Timestep & $\Delta t$ & $8.33 \times 10^{-3}$ \\
\quad Steps per frame & & 10 \\
\quad Drag coefficient & & 0.5 \\
\quad Smoothing factor & $\alpha_s$ & 0.955 \\
\midrule
\multicolumn{3}{l}{\textit{Particles}} \\
\quad Points per cell (default) & PPC$^{1/3}$ & 3 (${\sim}$37K particles) \\
\quad Points per cell (dragon) & PPC$^{1/3}$ & 4 (${\sim}$89K particles) \\
\bottomrule
\end{tabular}
\end{table}

The Lam\'e parameters correspond to Young's modulus $E = 140\,000$ and Poisson's ratio $\nu = 0.20$, modeling a stiff elastic solid.
The drag coefficient provides velocity damping at each timestep to stabilize the simulation.
The smoothing factor $\alpha_s$ controls exponential smoothing of the deformation gradient field across frames, with $\alpha_s = 0.955$ for the pairwise evaluation and $\alpha_s = 0.95$ for the teaser experiments.

\subsection{Optimization Parameters}
\label{sec:opt_params}

Table~\ref{tab:opt_params} summarizes the optimization and Chamfer schedule parameters.

\begin{table}[h]
\centering
\caption{Optimization and Chamfer schedule parameters.}
\label{tab:opt_params}
\begin{tabular}{lcc}
\toprule
Parameter & Symbol & Value \\
\midrule
\multicolumn{3}{l}{\textit{Gradient descent}} \\
\quad Number of frames & $T$ & 40 \\
\quad Forward--backward passes & & 3 \\
\quad GD iterations per pass & & 1 \\
\quad Line search iterations & & 5 \\
\quad Initial step size & $\alpha_0$ & 0.01 \\
\quad Adaptive step target norm & & 2\,500 \\
\midrule
\multicolumn{3}{l}{\textit{Chamfer schedule}} \\
\quad Chamfer weight & $w_{\text{ch}}$ & 10.0 \\
\quad Chamfer start frame & $t_{\text{start}}$ & 5 \\
\quad Chamfer ramp frames & $t_{\text{ramp}}$ & 10 \\
\quad Reverse weight (base) & $w_{\text{rev,base}}$ & 1.0 \\
\quad Reverse mode & & coupled (Eq.~\ref{eq:coupled_schedule}) \\
\quad Reverse gradient clamp ratio & $\kappa$ & 3.0 \\
\midrule
\multicolumn{3}{l}{\textit{Physics decay (Phase 2)}} \\
\quad Decay start frame & & 15 \\
\quad Decay ramp frames & & 5 \\
\quad Final physics weight & $w_{\text{phys,final}}$ & 0.1 \\
\quad Smoothing (Phase 2) & & 0.7 (from frame 15) \\
\bottomrule
\end{tabular}
\end{table}

The Chamfer loss is linearly ramped from frame~5 to frame~15 ($w_{\text{ch}}(t) = w_{\text{ch}} \cdot \min(1, (t - t_{\text{start}}) / t_{\text{ramp}})$ for $t \geq t_{\text{start}}$, zero otherwise).
The physics weight decays linearly from 1.0 to $w_{\text{phys,final}} = 0.1$ over frames 15--20, transitioning from physics-dominated to Chamfer-guided refinement.
The coupled reverse schedule (Eq.~\ref{eq:coupled_schedule}) ties $w_{\text{rev}}(t) = w_{\text{rev,base}} \cdot w_{\text{phys}}(t)$, ensuring that reverse (t$\to$s) pressure never exceeds the elastic resistance provided by the physics prior.

\paragraph{Dragon case study.}
The Sphere$\to$dragon experiment (Section~\ref{sec:dragon}) uses 4 particles per cell (${\sim}$89K particles) with a C++ inline bidirectional Chamfer gradient implementation for computational efficiency.
All other parameters remain identical to the default configuration.

\paragraph{Pairwise evaluation.}
The 20-pair evaluation (Section~\ref{sec:exp_pairwise}) uses 3 particles per cell with the same schedule.
For pairs with non-sphere sources, the source mesh is sampled to match the particle count of the corresponding sphere initialization.


\section{Additional Quantitative Results}
\label{sec:additional_results}

\subsection{Pairwise Source-to-Target Chamfer Distance}
\label{sec:pairwise_s2t}

Table~\ref{tab:pairwise_s2t} reports the s$\to$t component of Chamfer distance for all 20 directed morphing pairs.
Our method improves s$\to$t on all 20 pairs, confirming that the forward Chamfer improvement is universal across source geometries.

\begin{table}[h]
\centering
\scriptsize
\caption{Pairwise morphing results (s$\to$t CD) in matrix format.  Each cell shows Physics\,/\,\textbf{Ours}.  \textbf{Bold}: improvement over physics-only.  Our method improves s$\to$t on all 20 directed pairs.}
\label{tab:pairwise_s2t}
\setlength{\tabcolsep}{3pt}
\begin{tabular}{l ccccc}
\toprule
Src\,$\downarrow$\,/\,Tgt\,$\rightarrow$ & Sphere & Bunny & Duck & Cow & Teapot \\
\midrule
Sphere & ---                      & 0.181\,/\,\textbf{0.157} & 0.184\,/\,\textbf{0.173} & 0.178\,/\,\textbf{0.167} & 0.180\,/\,\textbf{0.148} \\
Bunny  & 0.199\,/\,\textbf{0.159} & --- & 0.186\,/\,\textbf{0.167} & 0.178\,/\,\textbf{0.159} & 0.187\,/\,\textbf{0.150} \\
Duck    & 0.181\,/\,\textbf{0.168} & 0.183\,/\,\textbf{0.178} & --- & 0.184\,/\,\textbf{0.176} & 0.190\,/\,\textbf{0.172} \\
Cow   & 0.186\,/\,\textbf{0.167} & 0.188\,/\,\textbf{0.156} & 0.325\,/\,\textbf{0.168} & --- & 0.193\,/\,\textbf{0.161} \\
Teapot & 0.176\,/\,\textbf{0.146} & 0.178\,/\,\textbf{0.165} & 0.273\,/\,\textbf{0.176} & 0.178\,/\,\textbf{0.169} & --- \\
\bottomrule
\end{tabular}
\end{table}

\subsection{2D Cross-Domain Collapse Experiment}
\label{sec:2d_collapse}

We provide full implementation details for the 2D collapse experiment presented in Figure~\ref{fig:2d_experiment}.

\paragraph{Setup.}
A source set of $N{=}600$ points uniformly sampled on a circle boundary ($r{=}0.8$, centered at the origin) is optimized to match a target star boundary ($M{=}200$ points, 5 arms, inner radius 0.25, outer radius 1.0).
We compare three optimization strategies, all running for 300 gradient steps with bidirectional Chamfer loss ($w_{\text{fwd}}{=}w_{\text{rev}}{=}1.0$) and 95th-percentile gradient clipping.

\begin{enumerate}
    \item \textbf{DCO} (per-point CD optimization, lr${=}0.015$): Each source point is independently moved along its Chamfer gradient. As predicted by Proposition~1, points collapse onto the nearest star vertices, leaving large target regions uncovered.
    \item \textbf{DCO + Repulsion} (lr${=}0.015$, $k$-NN with $k{=}6$, $\lambda_{\text{rep}}{=}2{\times}10^{-4}$): A local repulsive force ($1/d^2$) from the 6 nearest neighbors is added to the Chamfer gradient. Consistent with Proposition~3, the centroid drift is unchanged and collapse persists.
    \item \textbf{Shared-basis deformation} (lr${=}0.003$): The boundary is parameterized as $r(\theta) = a_0 + \sum_{k=1}^{K}[a_k \cos(k\theta) + b_k \sin(k\theta)]$ with $K{=}12$ Fourier modes. All points share the same coefficients $\{a_k, b_k\}$; the CD gradient backpropagates through the parameterization to update these shared coefficients. This provides global coupling analogous to how MPM couples particles through a shared Eulerian grid. $K{=}12$ is chosen to be the minimum number of modes that can represent a 5-arm star (5 concavities require at least $k{=}5$ harmonics; doubling to $K{=}12$ provides sufficient capacity while remaining low-dimensional).
\end{enumerate}

\paragraph{Metrics.}
We report three quantitative measures:
\begin{itemize}
    \item \textbf{Two-sided CD}: the standard bidirectional Chamfer distance between the optimized source and the target.
    \item \textbf{Coverage} (\%): the fraction of target points that have at least one source point within a distance threshold $\epsilon{=}0.05$.
    \item \textbf{Cluster fraction} (\%): the fraction of source points that are within $\epsilon{=}0.02$ of another source point, measuring the degree of many-to-one collapse.
\end{itemize}

\paragraph{Results.}
Table~\ref{tab:2d_results} summarizes the quantitative comparison.

\begin{table}[h]
\centering
\caption{Quantitative results of the 2D collapse experiment (circle $\to$ star).}
\label{tab:2d_results}
\begin{tabular}{lccc}
\toprule
Method & Two-sided CD $\downarrow$ & Coverage (\%) $\uparrow$ & Cluster Fraction (\%) $\downarrow$ \\
\midrule
DCO & 0.352 & 38 & 97 \\
DCO + Repulsion & 0.299 & 50 & 97 \\
Shared-basis (Fourier, $K{=}12$) & \textbf{0.123} & \textbf{62} & \textbf{81} \\
\bottomrule
\end{tabular}
\end{table}

DCO collapses nearly all source points onto the five-star vertices (cluster fraction 97\%), covering only 38\% of the target.
Adding local repulsion marginally improves coverage (50\%) but does not reduce the cluster fraction, confirming Proposition~3: translational invariance ensures the regularizer gradient sums to zero, so the net drift toward the target is unchanged.
The shared-basis deformation reduces two-sided CD by 65\% (0.352 $\to$ 0.123) and cluster fraction from 97\% to 81\%, demonstrating that global coupling through shared parameters effectively suppresses collapse, consistent with Corollary~1.

The corresponding visualizations are shown in Figure~\ref{fig:2d_experiment}.


\section{Additional Figures and Ablation}
\label{sec:additional_figures}

\subsection{Frame-Wise Chamfer Distance Convergence}
\label{sec:cd_convergence}

\begin{figure}[ht]
\centering
\includegraphics[width=\linewidth]{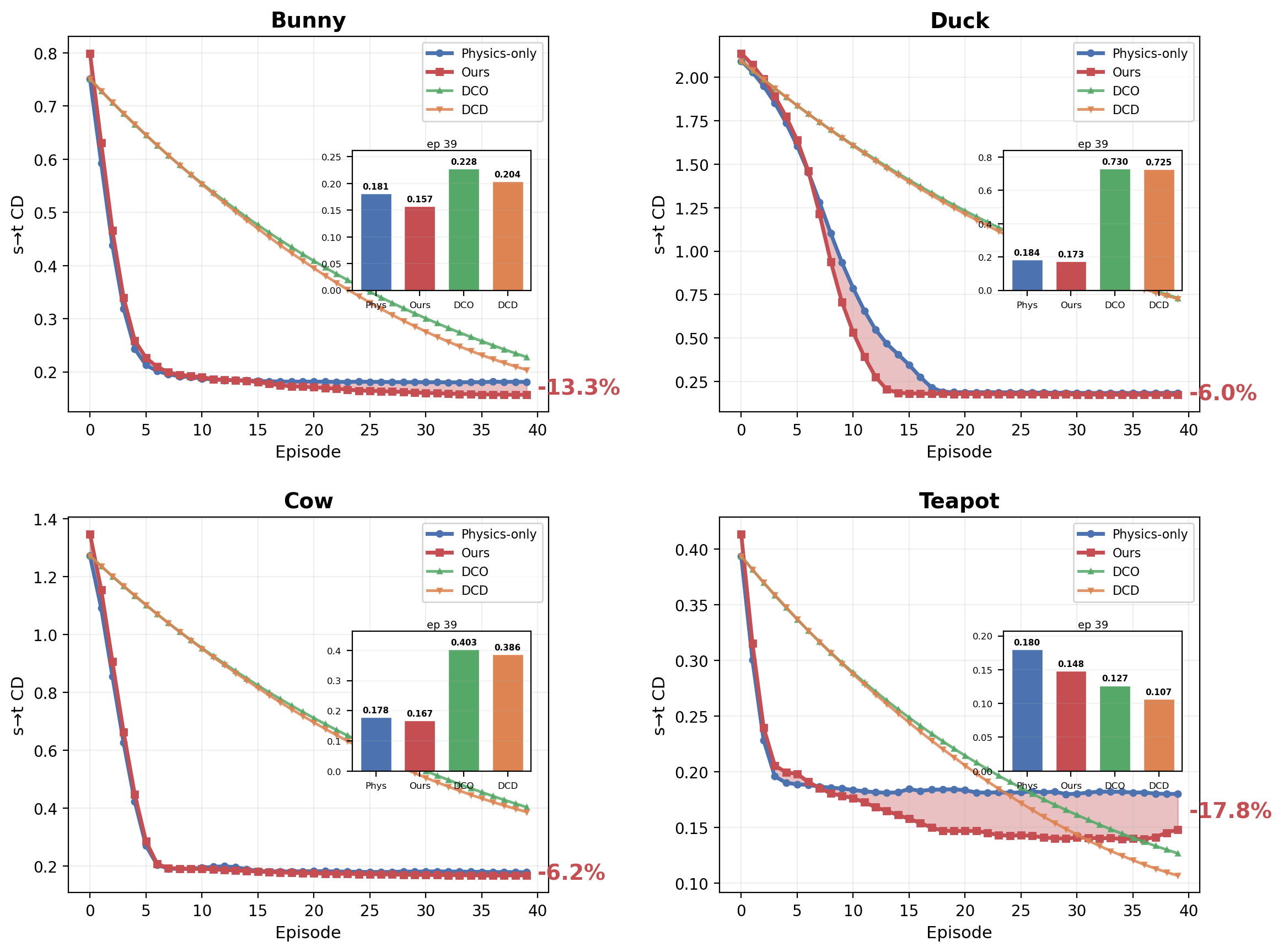}
\caption{Frame-wise source-to-target Chamfer distance across four target shapes. DCO and DCD converge to substantially worse plateaus on geometrically complex targets, while our method consistently refines the physics prior. On the nearly convex teapot, DCO achieves comparable s$\to$t but at the cost of structural collapse.}
\label{fig:CD_Loss_suppl}
\end{figure}

Figure~\ref{fig:CD_Loss_suppl} shows the s$\to$t Chamfer distance over 40 frames for all four target shapes. DCO and DCD plateau early at substantially worse values on geometrically complex targets, while our method continues to improve as the schedule transitions from physics-dominated to Chamfer-guided refinement.

\subsection{Trajectory and Internal Structure Comparison}
\label{sec:trajectory_suppl}

The trajectory and cross-section comparison between DCD and our method (Cow$\to$Duck) is presented in Figure~\ref{fig:trajectory}.
This comparison highlights the interior hollowing artifact unique to collapse-prone objectives, which global coupling prevents throughout the trajectory.

\subsection{Pairwise Scatter Plot}
\label{sec:pairwise_scatter_suppl}

\begin{figure}[ht]
\centering
\includegraphics[width=0.7\linewidth]{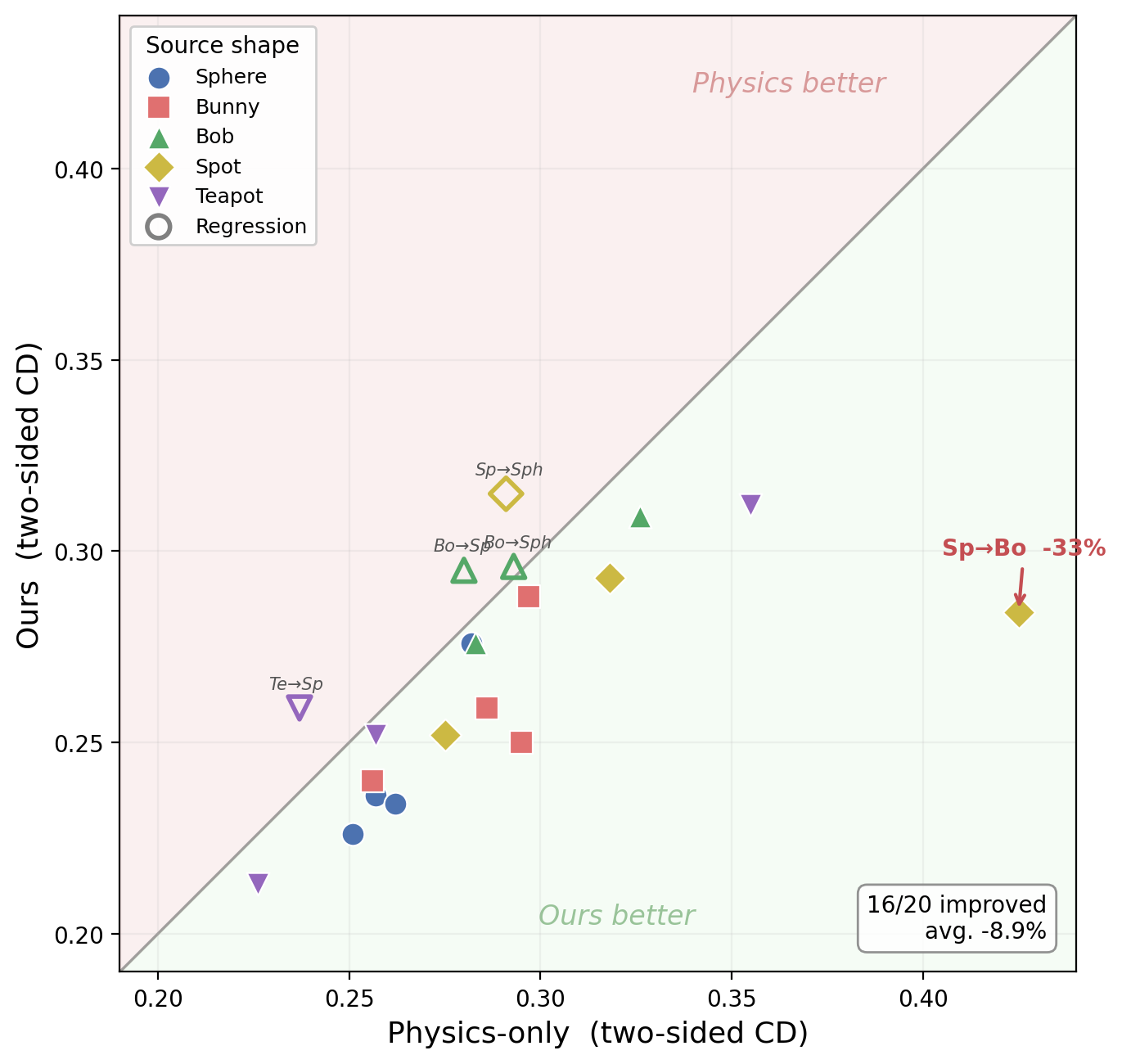}
\caption{\textbf{Pairwise two-sided CD: Physics-only vs.\ Ours.}
Each of the 20 directed pairs is one point; marker shape indicates the source.
Points below the diagonal (16/20) denote improvement by our method.
The four regressions all involve convex targets (Sphere or Cow) where the physics baseline already achieves near-optimal coverage.}
\label{fig:pairwise_scatter_suppl}
\end{figure}

Figure~\ref{fig:pairwise_scatter_suppl} visualizes the per-pair two-sided CD comparison across all 20 directed morphing pairs. Points below the diagonal indicate improvement by our method; the four regressions above the diagonal correspond to convex targets where the physics baseline already achieves near-optimal coverage.

\subsection{Ablation Study}
\label{sec:ablation_suppl}

\begin{table}[ht]
\centering
\scriptsize
\caption{Ablation study on Sphere$\to$bunny. Each row modifies one component from the full method.  Our default configuration achieves the best two-sided CD.}
\label{tab:ablation_suppl}
\begin{tabular}{lccc}
\toprule
Configuration & s$\to$t $\downarrow$ & t$\to$s $\downarrow$ & Two-sided $\downarrow$ \\
\midrule
Full method (Ours) & \textbf{0.157} & \textbf{0.163} & \textbf{0.226} \\
\midrule
\multicolumn{4}{l}{Coupled schedule} \\
\quad Fixed $w_{\text{rev}} = 1.0$ (no coupling) & 0.158 & 0.176 & 0.237 \\
\quad Fixed $w_{\text{rev}} = 0.5$ & 0.158 & 0.182 & 0.241 \\
\midrule
\multicolumn{4}{l}{Clamp ratio $\kappa$} \\
\quad $\kappa = 1$ & 0.162 & 0.184 & 0.245 \\
\quad $\kappa = 5$ & 0.156 & 0.179 & 0.237 \\
\midrule
\multicolumn{4}{l}{Physics decay timing} \\
\quad Decay at frame 15 & 0.158 & 0.179 & 0.239 \\
\quad Decay at frame 25 & 0.158 & 0.180 & 0.240 \\
\bottomrule
\end{tabular}
\end{table}

We ablate the key design choices of our method on the Sphere$\to$bunny pair (Table~\ref{tab:ablation_suppl}).
Most configurations fall within $\pm 0.003$ of the full method in s$\to$t, indicating broad robustness.
The main sensitivity is to the clamp ratio~$\kappa$: aggressive clamping ($\kappa{=}1$) discards gradient signal, degrading two-sided CD by $+$0.019, while relaxing to $\kappa{=}5$ slightly improves s$\to$t at the cost of t$\to$s.
The coupled reverse schedule shows clear asymmetry: a fixed $w_{\text{rev}}{=}0.5$ underperforms ($+$0.015) because weakened reverse pressure cannot counteract many-to-one collapse.
Physics decay timing between frames~15 and~25 has negligible effect ($\leq$0.014).

\end{document}